%% file: main.tex
\documentclass[12pt,letterpaper]{article}
\usepackage[margin=1in]{geometry}
\usepackage[utf8]{inputenc}
\usepackage{amsmath,amsfonts,amssymb}
\usepackage{mleftright,mathtools}
\DeclareMathAlphabet\mathbfcal{OMS}{cmsy}{b}{n}
\usepackage{amsfonts}
\usepackage{amssymb}
\usepackage{graphicx}
\usepackage[linesnumbered,ruled]{algorithm2e} 
\usepackage{float}
\usepackage{nicefrac}
\usepackage{mathrsfs}
\usepackage{pdfpages}
\usepackage{physics}
\usepackage{amsthm}
\usepackage{natbib}
\usepackage{multirow}
\usepackage{booktabs}
\usepackage{url}
\usepackage{hyperref}
\usepackage{placeins}
\usepackage{dsfont}
\usepackage{bm}
\usepackage{siunitx}
\usepackage{setspace}

\newtheorem{theorem}{Theorem}

\newtheorem{corollary}{Corollary}

\usepackage[dvipsnames]{xcolor}
\input{definitions.tex}

\begin{document}
\pagenumbering{gobble}
\setlength{\abovedisplayskip}{0.15cm}
\setlength{\belowdisplayskip}{0.15cm}
\pagestyle{empty}

\begin{titlepage}
\thispagestyle{empty}
\pagestyle{empty}
{\centering
\title{
\singlespacing \bf Bayesian meta-learning for modeling Alzheimer's disease progression}
\author{
\parbox{0.9\textwidth}{
\centering
Clara Hoffmann$^1$, Nadja Klein$^{1}$, for the Alzheimer’s Disease Neuroimaging Initiative$^{2}$\\[0.5em]
\footnotesize
$^1$Scientific Computing Center, Karlsruhe Institute of Technology, Germany\\[0.5em]
\vspace{-15pt}
\singlespacing
$^2$Data used in preparation of this article were obtained from the Alzheimer's Disease Neuroimaging Initiative (ADNI) database (adni.loni.usc.edu). As such, the investigators within the ADNI contributed to the design and implementation of ADNI and/or provided data but did not participate in the analysis or writing of this report. A complete listing of ADNI investigators can be found \href{ http://adni.loni.usc.edu/wpcontent/uploads/how_to_apply/ADNI_Acknowledgement_List.pdf}{here}.}}
\date{}
\maketitle }
\vspace{-20pt}
\begin{abstract}
\vspace{-12pt}
\singlespacing
Predicting whether an individual with Alzheimer's disease will experience mild or severe disease progression is essential for personalized treatment. Typically, practitioners seek to predict the distribution of a discrete disease score, conditional on an individual's current MRI volume and their historical disease trajectory. Classical statistical regression models and single-task neural networks are not well-suited for this purpose because fitting separate models is infeasible (since each individual typically has few observations), while ignoring individual-level correlation leads to poor generalization. Meta-learning, in contrast, provides a natural avenue to dynamically predict distributions without retraining and model nonlinear relationships between the outcome and covariates.
Motivated by this, we propose a Bayesian meta-learner that is trained on multiple individuals but tailors the predictive disease score distribution to each individual's historical data. 
Our model predicts on unseen individuals without retraining, scales linearly with the number of historical observations, and is guaranteed to be less overconfident when predicting long-term disease scores compared to its deterministic counterpart. On real-world data from the Alzheimer’s Disease Neuroimaging Initiative (ADNI) database, our model achieves performance competitive with both single-task models and deterministic meta-learners, while substantially improving performance when predicting long-term disease progression.
\end{abstract}
\vspace{-0.5cm}

\singlespacing
\vspace{0.5cm}\\
 \noindent
{\footnotesize{{\bf Corresponding author}: Prof.~Dr.~Nadja Klein, Scientific Computing Center, Karlsruhe Institute of Technology, Zirkel 2, 76131 Karlsruhe, Germany, nadja.klein@kit.edu.}}

\end{titlepage}
\clearpage
\pagenumbering{arabic}
\pagestyle{plain}
\section{Introduction}
Alzheimer’s disease (AD) is the prevailing cause of dementia and a major global health challenge.  Predicting AD progression is critical for understanding the underlying causes of the disease, administering early treatment, and allowing individuals and their families to prepare for the future. Instead of predicting the entire future progression, practitioners are often interested in predicting the distribution of a discrete disease severity score for an individual at a given time point from the respective MRI volume. To account for the strong heterogeneity between disease trajectories, it is useful to condition this distribution on an individual's historical disease trajectory, consisting of past MRI volumes and associated scores. A particular desideratum for this distribution is that it is well-calibrated when predicting ``long-term'' disease progression, where a substantial temporal gap exists between the time point at which to predict and the historical disease data.

Simultaneously modeling this disease score distribution across multiple individuals is challenging for both statistical and deep learning models.
Classical statistical regression models are not well-suited because individuals typically have few, irregularly spaced observations, making separate regressions infeasible. Mixed-effect models account for the correlation of an individual over time via random effects \citep{MccSeaNeu2008}, but if new individuals have few observations, the random effects remain close to the population mean. Moreover, the MRI volumes are high-dimensional, leading to scalability issues for regression in general\footnote{For example, stacking a standard MRI volume with dimension $256 \times 256 \times 256$ results in a covariate vector with dimension $> 10^6$ and equally as many parameters to be estimated in a linear regression model.}. While tensor regression approaches partially address this, they generally do not accommodate arbitrary nonlinear effects. 
Neural networks, on the other hand, can in principle model many practically relevant functions up to arbitrary accuracy \citep{Hor1991}. Yet, they are typically constructed as \emph{single-task} models, which in our application, correspond to predicting a single disease score from a single MRI volume. In this formulation, they cannot account for irregular time series without further adaptation, need to be retrained when new observations become available, and tend to produce overconfident predictions. In particular, they often assign nearly all probability mass to a single class when the covariate values lie far beyond the ones in the training data \citep{HeiAndBit2019}. 
This issue arises from the widespread use of piecewise-linear activation functions in modern neural networks, such as Rectified Linear Units \citep[ReLU;][]{Aga2018}, and can be mitigated by performing Bayesian inference over (a subset of) the network parameters \citep{KriHeiHen2020}.

To address these shortcomings, we propose a Bayesian meta-learner that combines shared learning of the disease score based on a current time point and MRI volume with a Bayesian last layer, which adapts to the historical disease trajectory via a hypernetwork.
From a statistical perspective, this can be viewed as a Bayesian regression model in which both the coefficients and covariates are learned from the data, and the coefficient posterior is dynamically adapted to each individual during prediction. 

Overall, we make three main contributions. First, we theoretically establish that modeling long-term disease scores with deterministic ReLU classifiers leads to overconfidence both in single-task and meta-learning, and empirically verify that this can be mitigated by performing Bayesian inference over the last-layer weights. Second, we show that accounting for an individual's historical disease trajectory via meta-learning improves predictive accuracy compared to single-task models. Third, we demonstrate on real-world data from the Alzheimer’s Disease Neuroimaging Initiative \citep[ADNI;][]{PetAisBec2010} database that our Bayesian meta-learner is competitive with other established methods for modeling disease trajectories, while yielding better performance for predicting long-term disease progression.

The remainder of the paper is organized as follows. Section~\ref{sec:background} provides background on AD, the data used to perform the benchmark study, and relevant models for disease progression. Section~\ref{sec:methodology} details our Bayesian meta-learning approach. Section~\ref{sec:estimation} presents the theoretical results on overconfidence and validates them in a simulation study. Section~\ref{sec:application} contains a detailed analysis with the real-world ADNI data, and Section \ref{sec:discussion} concludes with a discussion. All code is available in the Supplementary Material.
\section{Background}\label{sec:background}
\subsection{Alzheimer's disease}
AD is a neurodegenerative disease characterized by loss of brain volume, also known as cerebral atrophy \citep{KilMosAlb1993, LehBauChi1994}. This atrophy can be measured non-invasively with T1-Weighted MRI volumes \citep{FisFitNob2018}. 
Detecting the onset of AD is particularly challenging because the disease evolves over decades, during which pathological changes occur long before clinical symptoms appear \citep{JacLowWei2009, BucMinZet2012}. 
Identifying mild and moderate AD is of particular interest for testing drugs that may prevent progression from mild cognitive impairment to AD dementia \citep{SpeJacAis2011}. Treatment should be administered early to those individuals who are likely to experience a severe trajectory. Therefore, it is a key objective to predict an individual’s future disease progression from limited available data.

Clinical progression of AD is commonly measured using cognitive and functional tests, which summarize disease severity as a scalar, discrete score. The recommended score \citep{CedJarHer2013} to model all AD stages is the Clinical Dementia Rating-Sum of Boxes \citep[CDR-SB;][]{BerMilSto1988}. The CDR-SB is based on the Clinical Dementia Rating \citep[CDR;][]{HugBerDan1982}, 
which measures performance across six domains (cognition: memory, orientation, judgment/problem solving; function: community affairs, home/hobbies, personal care) by means of an interview with the individual and a qualified companion. Summing the scores on each category delivers the CDR-SB. It is a discrete score ranging from 0 to 78, where higher scores indicate higher disease severity. 

\subsection{Data}
Our methodological developments are motivated by the need for reliable and efficient assessment of real-world disease trajectories from the ADNI database. This database is a large-scale longitudinal study that tracks biomarkers and MRI volumes in cognitively normal individuals (normal controls, NC), those with mild cognitive impairment (MCI), and individuals with AD\footnote{Up-to-date information can be found under \url{adni.loni.usc.edu}}. A key objective of ADNI is to quantify progression between these clinical stages using serial MRI,  biological markers, as well as clinical and neuropsychological assessments. 

For our analysis, we use MRI volumes together with the corresponding CDR-SB score. Because some CDR-SB values occur only rarely, we group them into three clinically meaningful categories (NC, MCI, AD) to ensure a sufficient sample size per class, following a similar setup to \citet{ObrWarCul2008}. Figure~\ref{fig:mri_progression} illustrates representative averaged MRI slices for the three clinical stages with CDR-SB scores of 0 (left; NC),  $> 1$ and $\leq 5$ (center; MCI), and $> 5$ (right; AD). 
\begin{figure}[ht!]
    \centering
    \includegraphics[width=.8\linewidth]{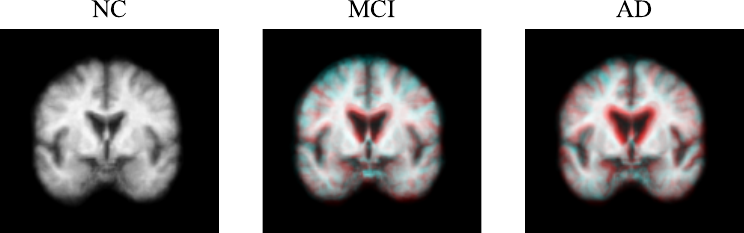}
    \caption{Averaged axial MRI slice (across $10$ individuals\protect\footnotemark) from the ADNI database for normal controls (left; NC), individuals with mild cognitive impairment (center; MCI), and individuals with Alzheimer's disease (right; AD). Red regions indicate brain atrophy, and blue regions indicate new brain tissue.}
    \label{fig:mri_progression}
\end{figure}
\footnotetext{Averaging across individuals is required due to ADNI’s Data Use Agreement,
which prohibits the publication of single-participant MRI slices.}

The images have been skull-stripped, intensity normalized, and projected to a standard coordinate space with the \emph{neural pre-processing} package \citep[NPPY;][]{HeWanSab2023}.
To highlight changes in brain matter, the MCI and AD slices are overlaid on the NC slice. Regions that are darker than the NC slice appear in red, indicating cerebral atrophy, while regions that are brighter appear in blue and highlight additional brain tissue.
Progressive cerebral atrophy is visible from NC to MCI to AD as indicated by the volume loss at the brain's center, where the ventricles (dark regions) are enlarged, and at the edges of the brain, where the brain's folds (cortical sulci) are widened \citep[see][for an introduction to cerebral atrophy interpretation]{BlinWei2021}.

\paragraph*{Train, validation, and test split}
For training, validation, and in-sample testing, we use data from the ADNI1 study phase covering the years 2005--2008. The in-sample test set is used to evaluate whether the models capture meaningful patterns within the timespan from which the training data is drawn. In contrast, to assess the generalization and uncertainty calibration for modeling long-term disease progression, we additionally use two out-of-sample test sets with recent data from the study phases ADNI3 (2016--2022) and ADNI4 (2023--2025). Table~\ref{tab:train_test_split} summarizes the number of observations for the train, validation, and test sets. 
\begin{table}[h!]
    \caption{Number of observations for the train, validation, and test sets.}
        \label{tab:train_test_split}
        \centering
        \small
        \begin{tabular}{lcc cc cc}
            \toprule
             \cmidrule(lr){2-4} \cmidrule(lr){5-6}
             & \multirow{2}{*}{\textbf{Training}} & \multirow{2}{*}{\textbf{Validation}}  &  \multicolumn{3}{c}{\textbf{Test}}  \\ 
             \cmidrule(lr){4-6}
             & & & \textbf{In-sample} & \multicolumn{2}{c}{\textbf{Out-of-sample}} \\
              \midrule
             \cmidrule(lr){4-6}
            Study phase & ADNI1 & ADNI1 & ADNI1 & ADNI3 & ADNI4 \\
            \midrule
            No. of observations & 942 & 243 & 257  & 92 & 11  \\
            No. of individuals & 242 & 59 & 76 & 42 & 11 \\
            \bottomrule
        \end{tabular}
    \end{table}
The eight-year gap between the training and out-of-sample period allows us to predict at time points that are far from the historical disease trajectory, and thereby highlight the overconfidence issue we seek to study.
For the out-of-sample test set, we include only ``rollover'' participants with at least one previous observation in ADNI1. The intermediate study phase, ADNI2 (2010--2016), is excluded because it requires additional processing to match the image format in ADNI1 and has a shorter temporal distance to the in-sample data. For reproducibility and in line with the ADNI guidelines, we report the retrieval dates, subject IDs, and additional processing steps for ADNI3 and ADNI4 in Supplementary Material A.

\subsection{Related work}\label{subsec:relatedwork}
Both deep learning and statistical regression models have been used for modeling disease progression, but existing methods often cannot process historical disease trajectories in parallel or lack the flexibility to model sufficiently nonlinear relationships between the outcome and covariates.

Deep learning methods for disease progression often focus on modeling temporal sequences or the latent disease state. Architectures for temporal data include recurrent neural networks \citep[RNNs;][]{Elm1990} and their gated variants, long short-term memory networks \citep[LSTMs;][]{HocSch1997}. Both models typically require ordered, regularly spaced observations and, without modification, can only make predictions into the future \citep{BayXiaZha2017, Zha2019}. Our proposed approach, on the other hand, can also interpolate between observed data. Deep state-space models \citep{AlaVan2019} offer better interpretability by learning structured representations of latent physiological states that drive disease progression. However, just like RNNs and LSTMs, they process data sequentially, resulting in long training times and high computational cost. In contrast, our proposed approach overcomes this limitation by processing the historical disease trajectory in parallel, while still leveraging individual-specific information.

Statistical regression models for disease progression mostly employ tensor regression, an extension of vector-based regression to multi-dimensional arrays. Recent developments allow partial sharing of regression coefficients across groups \citep{SuiXuLi2025}, and efficient estimation techniques have been proposed for neuroimaging data \citep{ZhouLiZhu2013}. However, tensor regression generally does not allow for arbitrary nonlinear mappings from the covariates to the outcome. In contrast, our approach preserves the interpretability and efficient inference of a Bayesian linear model, while providing the flexibility to capture nonlinear relationships between disease scores and MRI volumes.
Recent work has explored combining statistical models with deep neural networks. This includes neural networks with random effects \citep{SimRos2023}, which typically set the random-effect coefficient to zero when predicting on new individuals, which effectively removes the individual-level effects. In contrast, our Bayesian meta-learner can incorporate such effects even for unseen subjects.

\section{Methodology}\label{sec:methodology}
\subsection{Notation and objective}
To formally introduce the longitudinal data structure at hand, assume we observe each individual $i=1,\ldots,N$ at $T_i$ individual-specific time points.  Let  $Y_{i,t_{i,j}}^{(c)}$ denote the random variable for the discrete disease score at time point $t_{i,j}\in\lbrace t_{i,1},\ldots,t_{i,T_i}\rbrace$ with realization $y_{i,t_{i,j}}^{(c)}\in\lbrace 1,\ldots,K \rbrace$, where $K$ is the number of possible disease score categories. 
Let $x_{i,t_{i,j}}^{(c)}$ contain the MRI volume $\bar{x}_{i,t_{i,j}}^{(c)}$ at time point $t_{i,j}$ for individual $i$. Instead of the raw MRI volume, we can work directly with a suitable lower-dimensional embedding, so that we can assume $\bar{x}_{i,t_{i,j}}^{(c)} \in \Reals^{\dimimg}$ for some $\dimimg>1$ and $x_{i,t_{i,j}}^{(c)} = (\bar{x}_{i,t_{i,j}}^{(c) \top},t_{i,j})^\top \in \Reals^{D +1}$. Details on how to obtain the embedding will be given in Section \ref{sec:application}. 
All observed covariate-label pairs for individual $i$ are collected in the historical disease trajectory $\mathcal{D}_i^{(c)} = \lbrace(x_{i,t_{i,j}}^{(c)}, y_{i,t_{i,j}}^{(c)})\rbrace_{j=1}^{T_i}$. Our objective is to predict the conditional distribution of the disease score $Y_{i,t_{i, T_i + 1}}^{(t)}, \ldots, Y_{i,t_{i, T_i + H_i}}^{(t)}$ at new time points $t_{i,T_{i} +1}, \ldots, t_{i,T_{i}+H_i}$ and MRI volumes $\bar{x}_{i,t_{i,T_{i}+1}}^{(t)}, \ldots, \bar{x}_{i,t_{i,T_{i}+H_i}}^{(t)}$ 
\begin{align*}
Y_{i,t_{i,T_i+h_i}}^{(t)} \mid x_{i,t_{i,T_i + h_i}}^{(t)}, \D_i^{(c)}, \quad h_i = 1, \ldots, H_i,
\end{align*}
where $H_i > 0$ denotes the individual-specific prediction horizon.  The observed time points $t_{i,1}, \ldots, t_{i,T_i}$ and new time points $t_{i,T_i+1}, \ldots, t_{i, T_i+H_i}$ have no particular order, so that we may predict in the past, future, or in between observed data. This is a key advantage compared to LSTMs and RNNs (compare Section~\ref{subsec:relatedwork}), which can only predict on sequential data without further modification. Following common meta-learning terminology \citep{GarRoMa2018}, we refer to the observed disease trajectory  $\D_i^{(c)}$ as the \textit{context set} and the MRI volumes and time points for prediction $\D_i^{(t)} = \lbrace x_{i,t_{i,j}}^{(t)}\rbrace_{j=T_i +1}^{T_i + H_i}$ as the \textit{target set}. Jointly, the context and target sets form a so-called \textit{task} $\D_i = \D_i^{(c)} \cup \D_i^{(t)}$. 

\subsection{Meta-learning model}
In the following, we construct a meta-learner $f_{\theta}(x_{i,t_{i,j}}^{(t)}, \D_i^{(c)})$  with parameter $\theta\in\mathbb{R}^P$ which predicts the disease score distribution at any target covariate
$x^{(t)}\in\mathcal{X}$ conditional on the individual's history $\mathcal{D}_i^{(c)}\in\mathcal{H}$. More formally,
\[ f_\theta : \mathcal{X} \times \mathcal{H} \to (0,1)^K, \]
where
$\mathcal{X} = \mathbb{R}^{\dimimg +1}$
contains the MRI embedding and time point, and
$\mathcal{H} = \bigcup_{T \ge 1} \bigl(\mathcal{X} \times \mathcal{Y}\bigr)^T$ denotes the space of all finite individual-specific disease trajectories for $\mathcal{Y}=\lbrace 1,\ldots,K\rbrace$.  The meta-learner is only trained \textit{once} and can make predictions for new individuals and tasks without retraining.

\subsubsection{Architecture}\label{sec:architecture}
The architecture consists of three components: (1) a \emph{context encoder} that predicts individual-specific last-layer weights from the context set, (2) an \emph{aggregator} that maps the last-layer weights to a fixed-size matrix, and (3) a \emph{target encoder} that uses these weights to predict class probabilities at a given target covariate.

The context encoder $f_{\vartheta}^{(c)}: \mathcal{X}\times\mathcal{Y}\to \Reals^{ \numclasses \times \dimlastlayerweights }$ with parameters $\vartheta$ is a hypernetwork\footnote{A hypernetwork is a neural network that predicts the weights of another neural network, and we refer to \citet{HaDaiLe2016, ChaZhoLu2024} for an overview.} which predicts a weight matrix $W_{i,t_{i,j}}$ from each context point $(x_{i,t_{i,j}}^{(c)},y_{i,t_{i,j}}^{(c)}) \in \D_i^{(c)}$ via
\begin{align}\label{eq:context_encoder}
W_{i,t_{i,j}} & = f_{\vartheta}^{(c)}(x_{i, t_{i,j}}^{(c)}, y_{i,t_{i,j}}^{(c)}).
\end{align}
The weight matrices are then aggregated into a fixed-size matrix $W_i$ to be used in the last layer of the target encoder
\begin{align}\label{eq:aggregator}
    W_i = \bigoplus_{j=1}^{T_i} W_{i,t_{i,j}},
\end{align}
where $\bigoplus$ denotes an aggregation operator. The aggregation operator must be commutative to ensure permutation-invariance. Possible choices are the minimum, maximum, and we use the mean as the aggregator throughout\footnote{Given the progressive nature of AD, one could also assign greater weight to later time points. However, we decided against this to preserve the model's ability to interpolate.}.

The target encoder is a feedforward neural network denoted by $f_{\psi}^{(t)}: \mathcal{X}\times\Reals^{\numclasses  \times \dimlastlayerweights} \to (0,1)^\numclasses$ that predicts the class probabilities for the disease score given the target covariates $x_{i,t_{i,j}}^{(t)}$ and the aggregated weight matrices $W_i$.
The last layer of the target encoder at a target covariate $x_{i,t_{i,j}}^{(t)}$ is linear, followed by a softmax activation
\begin{align}\label{eq:target_encoder}
    f_{\psi}^{(t)}(x_{i,t_{i,j}}^{(t)}, W_i) &=
    (\mathbb{P}(Y_{i,t_{i,j}}^{(t)} = 1 \mid x_{i, t_{i,j}}^{(t)}, \D_i^{(c)}), \ldots,  \mathbb{P}(Y_{i,t_{i,j}}^{(t)} = \numclasses \mid x_{i,t_{i,j}}^{(t)}, \D_i^{(c)}))^\top  \nonumber \\
    &= a(W_i \phi^{(t)}(x_{i,t_{i,j}}^{(t)})),
\end{align}
where $\phi^{(t)}(x_{i,t_{i,j}}^{(t)})\in\mathbb{R}^{F}$ is the target embedding. The target embedding is a function of $\psi$, but we suppress this dependence for better readability. 
 In addition,
$a:\mathbb{R}^{\numclasses}\to(0,1)^{\numclasses}$ is the softmax activation, where
for a vector $v = (v_1, \ldots, v_{\numclasses})^\top \in \Reals^{\numclasses}$ the softmax activation vector $a(v) = (a_1(v), \ldots, a_\numclasses(v))^\top$ contains elements 
\begin{align*}
    a_k(v) = \frac{\exp(v_k)}{\sum_{l=1}^\numclasses \exp(v_l)}, \quad  \, k = 1, \ldots, \numclasses.
\end{align*} 
Combining all components, the meta-learner has parameters $\theta=\lbrace\vartheta,\psi\rbrace$ and is formally defined as
\begin{align}
f_\theta(x^{(t)}, \mathcal{D}_i^{(c)})\;=\; f_{\psi}^{(t)}\!\left( x^{(t)},\; \bigoplus_{(x',y')\in\mathcal{D}_i^{(c)}}f_{\vartheta}^{(c)}(x',y')\right).
\end{align}
Because the last-layer weights $W_i$ are obtained by forward passing $\mathcal{D}_i^{(c)}$ through the context encoder, the meta-learner can adapt to new individuals or tasks without retraining. 
While we do not impose any architectural constraints on the context encoder, we assume the target encoder has linear hidden layers with piecewise linear activation functions, such as ReLU. Such activation functions remain the standard choice as they make training fast and stable. Later in Section~\ref{sec:estimation}, we show that this assumption is a sufficient condition for the overconfidence issue of long-term predictions.
Our architecture employs standard elements of meta-learning models and follows a design similar to established architectures such as Conditional Neural Processes \citep[CNPs;][]{GarRoMa2018} and Conditional Neural Adaptive Processes (CNAP; \citealp{ReqGorBron2019}). Figure~\ref{fig:architecture} illustrates our proposed architecture.
\begin{figure}[H]
    \centering
    \includegraphics[width=.85\linewidth]{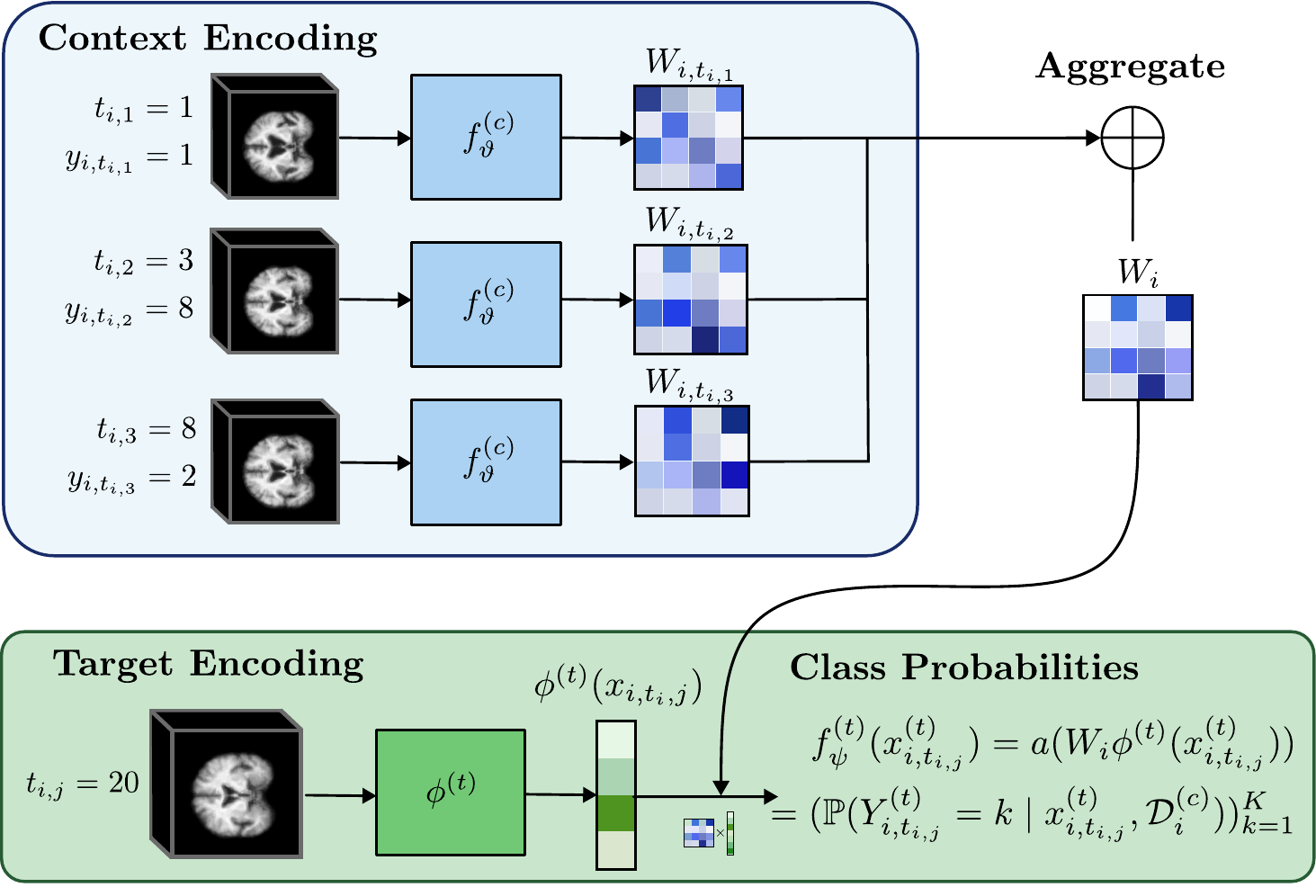}
    \caption{Illustration of the meta-learning architecture.}
    \label{fig:architecture}
\end{figure}

\subsubsection{Training}
We train the meta-learner $f_{\theta}$ to predict $y_{i,t_{i,j}}^{(c)}$ conditioned on the covariates $x_{i,t_{i,j}}^{(c)}$, given a randomly sampled subset of the context set of size $T_i^{(\tilde{c})}$, denoted as $\D_i^{(\tilde{c})}$, such that $\D_i^{(\tilde{c})} \subset \D_i^{(c)}$. This approach is known as \textit{episodic training} \citep{VinBluLil2016}.
The model parameters $\theta$ are chosen to maximize the
conditional log-likelihood of the full context set given the sampled subset
\[ \argmax_{\theta \in \mathbb{R}^P}
\;\mathbb{E}_{p(\mathcal{D}^{(\tilde{c})})}
\Bigl[\log \mathbb{P}\bigl(\mathcal{D}^{(c)} \mid \mathcal{D}^{(\tilde{c})}, \theta\bigr)\Bigr],\]
where  $ \D^{(c)} = \bigcup_{i=1}^N \D_i^{(c)} $ and  $ \D^{(\tilde{c})} = \bigcup_{i=1}^N \D_i^{(\tilde{c})} $ collects the data across all individuals and $\D^{(\tilde{c})}$ is distributed with probability density $p(\D^{(\tilde{c})})$.
In practice, the expectation over  $p(\D^{(\tilde{c})})$ is approximated via Monte Carlo integration, where $p(\D^{(\tilde{c})})$ is defined implicitly through a two-step sampling scheme. First, the number of observations in $\mathcal{D}^{(\tilde{c})}$ is drawn uniformly from a predefined interval. Then, the corresponding number of elements is drawn randomly and without replacement from $\mathcal{D}^{(c)}$. 
Learning can be performed by empirical risk minimization
\begin{align}\label{eq:mle}
    \hat{\theta} = \argmin_{\theta\in\mathbb{R}^P} \;-\, \mathbb{E}_{p(\mathcal{D}^{(\tilde{c})})}\left[\sum_{i=1}^{N}\sum_{j=1}^{T_i}\log \mathbb{P}\bigl(Y_{i,t_{i,j}}^{(c)} = y_{i,t_{i,j}}^{(c)}\mid x_{i,t_{i,j}}^{(c)},\mathcal{D}_i^{(\tilde{c})},\theta
\bigr)\right],
\end{align}
assuming conditional independence of observations given the subset of the context set and parameters, i.e.,  $Y_{i,t_{i,j}}^{(c)} \perp Y_{m,t_{m,l}}^{(c)} \mid \D^{(\tilde{c})}, \theta$ for $(i,j) \neq (m,l)$. We perform minimization using stochastic gradient descent (SGD).

The objective from Eq.~\eqref{eq:mle} yields a maximum-likelihood estimate (MLE). Similarly, a maximum-a-posteriori (MAP) estimate can be obtained by adding a weight regularization term $\lambda \|\theta\|_2^2$ to the objective in Eq.~\eqref{eq:mle}, where $\lambda > 0$. This corresponds to assuming a Gaussian prior on each parameter with variance inversely proportional to $\lambda$. In practice, we train all models without weight regularization and explain this choice in Section~\ref{sec:pred_inf}.
Algorithm~\ref{alg:training} lists the steps for obtaining a point (MLE) estimate $\hat{\theta}$ for the meta-learner parameters. 

\begin{algorithm}[H]
\caption{Training}\label{alg:training}
\SetKwInOut{Input}{Input}
\SetKwInOut{Output}{Output}
\Input{Context data $\mathcal{D}_1^{(c)},\ldots,\mathcal{D}_N^{(c)}$, initial parameters $\theta^{[0]}$, scalar learning rate $\eta$.}
\Output{Point estimate $\hat\theta=\lbrace\hat\vartheta,\hat\psi\rbrace$.}

Initialize $\theta \leftarrow \theta^{[0]}$ and set $s=0$. \\
\While{loss has not converged}{
    $s \leftarrow s+1$.\\
    Sample a minibatch $B \subset \{1,\ldots, N\}$. \\
    \ForEach{$i \in B$}{
        Sample a random subset $\mathcal{D}_i^{(\tilde{c})} \subset \mathcal{D}_i^{(c)}$. \\
        Compute the task-level negative log-likelihood:
        $
        \ell_i^{[s]}
        \coloneqq
        -\frac{1}{T_i}
        \sum_{j=1}^{T_i}
        \log
        \mathbb{P}\bigl(
        Y_{i,t_{i,j}}^{(c)} = y_{i,t_{i,j}}^{(c)}
        \mid
        x_{i,t_{i,j}}^{(c)},\mathcal{D}_i^{(\tilde{c})},\theta^{[s-1]}
        \bigr).
        $
    }
    Update parameters via SGD:
    $
    \theta^{[s]}
        \leftarrow
    \theta^{[s-1]}
    - \eta \,
      \nabla_{\theta}
      \frac{1}{|B|}
      \sum_{i\in B} \ell_i^{[s]}\big|_{\theta = \theta^{[s-1]}}.
    $
}
\Return{$\hat\theta=\theta^{[s]}$.}
\end{algorithm}
Once $\hat{\theta}$ has been obtained, a point estimate $\widehat{W}_i$ of the individual-specific weights can be obtained by passing the context set $\D_i^{(c)}$ through the context encoder and aggregator.

\subsection{Predictive inference}\label{sec:pred_inf}
We now present two approaches for obtaining predictions with the meta-learning architecture, which we henceforth refer to as the deterministic and Bayesian meta-learners. The deterministic meta-learner employs the point estimate $\widehat{W}_i$ for prediction, while the Bayesian meta-learner marginalizes over the full task-specific posterior $p(W_i \mid \D_i^{(c)})$ for each individual.
We introduce both approaches in detail, covering both point and distributional predictions. While we describe prediction for individuals $i = 1, \ldots, N$, the approaches can also be used to directly predict on new individuals. 

\subsubsection{Deterministic prediction}\label{sec:det_pred}
The deterministic meta-learner uses the point estimate $\widehat{W}_i$ of the individual-specific last-layer weight matrix without accounting for weight uncertainty.
The predictive probability for the deterministic meta-learner at class $k\in\{1,\ldots,K\}$ at a target covariate $x_{i,t_{i,j}}^{(t)}$ is
\begin{align}\label{eq:det_predictor}
    \hat{p}_{i,t_{i,j}}^{\mathrm{Det}}(k)
    = a_k\!\bigl(\,\widehat{W}_i \phi^{(t)}(x_{i,t_{i,j}}^{(t)})\bigr),
\end{align}
which provides an estimate of $\mathbb{P}(Y_{i,t_{i,j}}^{(t)} = k \mid x_{i,t_{i,j}}^{(t)},\mathcal{D}_i^{(c)})$ for individual $i=1,\ldots,N$ and time point $j=T_i +1, \ldots, T_i + H_i$.
The corresponding point prediction is obtained by selecting the class with the highest predicted probability, $\hat{y}_{i,{t_{i,j}}}^{\text{Det}} = \argmax_{k \in \{1,\ldots, K\}} \hat{p}_{i,{t_{i,j}}}^{\text{Det}}(k)$.

\subsubsection{Bayesian inference}\label{sec:bayes_pred}
As we will illustrate theoretically and empirically in Section~\ref{sec:estimation}, the target encoder $f_{\psi}^{(t)}$ will realize arbitrarily high confidence when predicting at time points far away from those in the context set. To mitigate this effect, we treat $W_i$ as a random variable and integrate over its posterior $p(W_i \mid \D_i^{(c)})$. This approach is known as last-layer inference \citep{SnoRipSwe2015, PinGorNal2019, KleNot2020, HofKle2023} and, for our purpose, the last-layer Laplace approximation \citep{RitBotBar2018} is an attractive choice to estimate the posterior as it provides a low-cost approximation without gradient update steps. \\
As a preliminary, note that the likelihood of the context set for individual $i = 1, \ldots, N$ is
\begin{align*}
   p(\D_i^{(c)} \mid W_i) = \prod_{j=1}^{T_i}\mathbb{P}(Y_{i,t_{i,j}}^{(c)} = y_{i,{t_{i,j}}}^{(c)} \mid W_i, x_{i,t_{i,j}}^{(c)}).
\end{align*}
To perform Bayesian inference, we specify a prior over the last-layer weights $W_i$. A standard choice is an isotropic matrix Gaussian prior $W_i \sim \MN(0,\sigma_w^2 I_\numclasses, I_\dimlastlayerweights)$ with a prior variance factor $\sigma_w$.

\paragraph*{Laplace approximation}
The core idea of the Laplace approximation is to approximate the posterior $p(W_i \mid \D_i^{(c)})$ with a Gaussian distribution, which is centered at the MAP estimate and whose covariance matches the local curvature of the posterior at that point. Compared to estimating an exact, potentially intractable posterior, the Laplace approximation is computationally cheap and can be fitted fast across many small datasets. 

The Laplace approximation stems from a second-order Taylor approximation of the log-posterior around the MAP estimate. For notational convenience, let $w_i \coloneqq \text{vec}(W_i)$ and its MAP estimate $\widehat{w}_i \coloneqq \text{vec}(\widehat{W}_i)$, where $\text{vec}(M)$ transforms a matrix $M$ into a column vector by stacking the columns of $M$ on top of each other.
Expanding the log-posterior around $\widehat{w}_i$ yields
\begin{align}\label{eq:taylor_approx}
    \log p(w_i \mid \D_i^{(c)}) \approx \log p(\widehat{w}_i \mid \D_i^{(c)}) - \frac{1}{2}(w_i - \widehat{w}_i)^\top H (w_i - \widehat{w}_i),
\end{align}
where $H \in \Reals^{(K \cdot F) \times (K \cdot F)}$ is the Hessian of the negative log-posterior  evaluated at the MAP
\begin{align*}
H &= - \nabla^2_{w_i} \log p(w_i \mid \mathcal{D}_i^{(c)}) \;\Big|_{w_i = \widehat{w}_i} = - \nabla^2_{w_i} 
\Big[ \log p(\mathcal{D}_i^{(c)} \mid w_i) + \log p(w_i) \Big] 
\;\Big|_{w_i = \widehat{w}_i}.
\end{align*}
Note that there is no linear term in Eq.~\eqref{eq:taylor_approx} because we assume the gradient of the log posterior to be zero at the MAP estimate, that is $\nabla_{w_i} \log p(w_i \mid \mathcal D_i^{(c)})\big|_{w_i = \widehat w_i} = 0
$.
The posterior is then multivariate Gaussian with $p(w_i \mid \D_i^{(c)}) = \mathcal{N}(w_i; \widehat{w}_i, H^{-1})$.
The distribution of $W_i$ can be expressed as a matrix Gaussian distribution using Kronecker factors $A_i^{-1}$, $B_i^{-1}$ of the inverse Hessian, i.e., $H_i^{-1} \approx A_i^{-1} \otimes B_i^{-1}$, resulting in
\begin{align*}
    W_i \mid \D_i^{(c)} \sim \MN_{}(\widehat{W}_i, B_i, A_i) \Leftrightarrow \text{vec}(W_i)\mid \D_i^{(c)}  \sim \mathcal{N}(\text{vec}(\widehat{W}_i), A_i \otimes B_i).
\end{align*}
More details on the Kronecker factors can be found in \citet{BotRitBar2017}, and we compute the factors using the \emph{backpack} package by \citet{DanKunHen2019}.
We center the Laplace approximation at the MLE from Eq.~\eqref{eq:mle}, which is equivalent to using an improper Gaussian prior with infinite variance during training. Following standard practice \citep{HenMazDiet2019, KriHeiHen2020}, the prior variance for prediction is subsequently determined using a small fraction of in-sample and out-of-sample data (see Supplementary Material B for more details).

\paragraph*{Posterior predictive distribution}
To predict on the target set, we leverage the fact that the distribution of the pre-activation outputs is a $K$-dimensional Gaussian
\begin{align}\label{eq:dist_activations}
    p(W_i\phi^{(t)}(x_{i,t_{i,j}}^{(t)}) \mid \D_i^{(c)}, x_{i,t_{i,j}}^{(t)}) = \mathcal{N}(W_i \phi^{(t)}(x_{i,t_{i,j}}^{(t)}) ;  \widehat{W}_{i}\phi^{(t)}(x_{i,t_{i,j}}^{(t)}),  (\phi^{(t)}(x_{i,t_{i,j}}^{(t)})^\top A_i \phi^{(t)}(x_{i,t_{i,j}}^{(t)}))B_i).
\end{align} 
This distribution is computationally more efficient to sample from than the $F \times K$-dimensional matrix Gaussian distribution of $W_i$.
We approximate the posterior predictive class probabilities $\mathbb{P}(Y_{i,t_{i,j}}^{(t)} = k \mid x_{i,t_{i,j}}^{(t)}, \D_i^{(c)})$ based on $M$ Monte Carlo samples of the pre-activation outputs $\lbrace(W_i\phi^{(t)}(x_{i,t_{i,j}}^{(t)}))^{[1]}, \ldots,  (W_i\phi^{(t)}(x_{i,t_{i,j}}^{(t)}))^{[M]} \rbrace$ via
\begin{align}\label{eq:mc_llla}
    \hat{p}_{i,{t_{i,j}}}^{\text{MC}}(k) = \frac{1}{M} \sum_{m=1}^M
    a_k\big((W_i \phi^{(t)}(x_{i,t_{i,j}}^{(t)}))^{[m]}\big),
\end{align}
for individual $i=1,\ldots,N$, time points $j=T_i +1,\ldots, T_i + H_i$, and class $k=1,\ldots,K$.
As before, point predictions are obtained by selecting the class with the highest predicted probability $\hat{y}_{i,{t_{i,j}}}^{\text{MC}} = \arg \max_{k \in \{1,\ldots,K\}} \hat{p}_{i,{t_{i,j}}}^{\text{MC}}(k)$.
Algorithm~\ref{alg:inference} summarizes the steps for inference.

\begin{algorithm}[H]
\caption{Predictive inference for the Bayesian meta-learner}\label{alg:inference}
\SetKwInOut{Input}{Input}
\SetKwInOut{Output}{Output}
\Input{Context set $\mathcal{D}_i^{(c)}$, target covariate $x_{i,t_{i,j}}^{(t)}$.}
\Output{Posterior predictive class probabilities $\mathbb{P}(Y_{i,t_{i,j}}^{(t)} = k \mid x_{i,t_{i,j}}^{(t)}, \D_i^{(c)})$.}
Obtain MAP estimate $\widehat{W}_i$ of the last-layer weights from $\D_i^{(c)}$ using Eq.~\eqref{eq:context_encoder} and~\eqref{eq:aggregator}. \\
Evaluate target embeddings at the context covariates $\phi^{(t)}(x_{i,t_{i,j}}^{(c)})$ for $j = 1, \ldots, T_i$. \\
Approximate the posterior $p(W_i \mid \D_i^{(c)})$ with a last-layer Laplace approximation. \\
Obtain $M$ samples $(W_i\phi^{(t)}(x_{i,t_{i,j}}^{(t)}))^{[1]}, \ldots,  (W_i\phi^{(t)}(x_{i,t_{i,j}}^{(t)}))^{[M]}$  using Eq.~\eqref{eq:dist_activations}. \\
Compute $\mathbb{P}(Y_{i,t_{i,j}}^{(t)} = k \mid x_{i,t_{i,j}}^{(t)}, \D_i^{(c)})$ using Eq.~\eqref{eq:mc_llla}.
\end{algorithm}
Last-layer inference only adds a small additional cost of 1--2 seconds per individual on an Apple M1 processor. The computational complexity is $\mathcal{O}(T_i F^2)$ for constructing the Hessian, $\mathcal{O}(K F^3)$ for its inversion, and $\mathcal{O}(K F^2)$ for evaluating posterior variances.

\section{Properties and simulation study}\label{sec:estimation}
We next study the properties of the deterministic meta-learner and its Bayesian counterpart. By building on previous results from single-task learning, we show theoretically that the deterministic meta-learner is overconfident when predicting long-term disease progression, in the sense that nearly all probability mass is assigned to a single class.
As a practical example, consider an individual whose disease scores and MRI volumes were collected at multiple time points between 2005 and 2008, forming $\D_i^{(c)}$. In 2025, a new MRI scan, yielding $\bar{x}_{i,t_{i,T_i + 1}}^{(t)}$ at $t_{i,T_i + 1}=2025$ is recorded, creating the target set $\D_i^{(t)}$. The CDR-SB score is not collected due to time constraints, and the individual seeks a prediction instead. With the 17-year gap between context and target set, the deterministic meta-learner may place nearly all probability mass on a single class, leading to poor calibration and making the model unsuitable for clinical applications.

We verify these theoretical results in a simulation study and demonstrate that the Bayesian inference scheme from Section \ref{sec:pred_inf} mitigates this overconfidence and improves calibration.

\subsection{Properties}
We analyze the limit of the predictive class score $\lim_{t_{i,j} \rightarrow \infty} \mathbb{P}(Y_{i, t_{i,j}}^{(t)} = k \mid \bar{x}_{i,t_{i,j}}^{(t)}, t_{i,j}, \D_i^{(c)})$ when predicting at target time points $t_{i,j}$ that are far from those in the context set, where $j = T_i +1, \ldots, T_i + H_i$ and $k = 1, \ldots, K$, which is equivalent to convergence in distribution. While $t_{i,j} = \infty$ is never observed in practice, we demonstrate on synthetic and real data that the asymptotic distribution is already attained at finite, practically relevant time horizons, such as a few years beyond the time points in the context set.

\subsubsection{Overconfidence with deterministic single-task learning}
We begin by recalling a known result by \citet{HeiAndBit2019} that, under mild assumptions, single-task\footnote{This corresponds to training a model which ignores all context information and predicts the disease score directly from the current MRI image and time, i.e. $\mathbb{P}(Y_{i,t_{i,j}}^{(t)} = k \mid x_{i,t_{i,j}}^{(t)})$.} ReLU classifiers become arbitrarily confident when their inputs are scaled by a sufficiently large factor $\alpha > 0$.
In other words, for inputs of the form $\alpha x$, the predictive class probability tends to one for a single class. 
We then show that this effect persists even when only a subset of the inputs is scaled, as in our use case, where time is scaled, but the MRI volume is not. This yields a new corollary, which we then extend to the meta-learning setting.

Consider a single-task feedforward neural network $g: \mathbb{R}^{J} \rightarrow \mathbb{R}^{K}$ with ReLU activations in the hidden layers. 
Such a network can be expressed as an affine function (a linear function plus a constant vector), on a finite partition of the input space into polytopes $\{Q_l\}_{l=1}^R$ such that $\mathbb{R}^{J}= \cup_{l=1}^R Q_l$. These polytopes are referred to as \emph{linear regions} \citep{AroBasMi2016}, since on each such region the network can be expressed as an affine function. 
For any $x \in Q_l$, the network output can be written as
\begin{align*}
    g\big|_{Q_l}(x) = V_l x + a_l.
\end{align*}
This is useful because it allows for analyzing a linear function within each linear region rather than a nonlinear one on the entire input space. In addition, almost all inputs remain in the same linear region when scaled by a sufficiently large scalar $\alpha > 0$ \cite[Lemma 3.1]{HeiAndBit2019}, so that $V_l$ and $a_l$ remain unchanged. We provide further information on the linear regions in Supplementary Material C.
The overconfidence problem can be formalized by the following theorem.
\begin{theorem}[\citeauthor{HeiAndBit2019}, \citeyear{HeiAndBit2019}]\label{thm:hein}
Let $\mathbb{R}^{J} = \cup_{l=1}^R Q_l$ and $g\big|_{Q_l}(x) = V_l x + a_l$ be the piecewise affine representation of the output of a ReLU network on $Q_l$. Suppose that $V_l$ does not contain any identical rows for all $l = 1, \ldots, R$, then for almost any $x \in \mathbb{R}^{J}$ and any $\epsilon > 0$, there exists an $\alpha > 0$ and a class $k \in \{1, \ldots, K\}$ such that it holds $a_k(g(\alpha x)) \geq 1 - \epsilon$.
Moreover, $\lim_{\alpha \rightarrow \infty} a_k(g(\alpha x)) = 1$.
\end{theorem}
This result can be extended to partially scaled inputs via the following corollary.
\begin{corollary}\label{cor:overconfidence}
Let the assumptions of Theorem~\ref{thm:hein} hold and let in addition $x \in \mathbb{R}^{J}$ be partitioned into two disjoint subsets $x = [x_{1:I}^\top, x_{(I+1):J}^\top]^\top$ for some $I \in \{1, \ldots ,J-1\}$. Then for almost any $x \in \mathbb{R}^{J}$ and any $\epsilon > 0$, there exists an $\alpha > 0$ and a class $k \in \{1, \ldots, K\}$ such that it holds $ a_k(g([\alpha x_{1:I}^\top, x_{(I+1):(J)}^\top]^\top)) \geq 1 - \epsilon$.
Moreover, $\lim_{\alpha \rightarrow \infty} a_k(g([\alpha x_{1:I}^\top,  x_{(I+1):(J)}^\top]^\top)) = 1$.
\end{corollary}
\begin{proof}[Proof]
Consider the pre-activation output $z(\cdot)$ of a neuron in the first layer of $g$
\begin{align*}
    z([\alpha x_{1:I}^\top, x_{(I+1):J}^\top]^\top)
    & = [W_1^{(I)}; W_1^{(J - I)}] [\alpha x_{1:I}^\top; x_{(I+1):(J)}^\top]^\top + b \\
    & = W_1^{(I)} (\alpha x_{1:I}) + \underbrace{W_1^{(J - I)} x_{(I+1):(J)} + b}_{=: b_*},
\end{align*}
where $W_1 = [W_1^{(I)}; W_1^{(J  - I)}]$ are the first layer weights partitioned in the same way as $x$, and $b$ is the bias. The term $W_1^{(J -I)} x_{(I+1):(J)}$  may be regarded as an additional bias term $b_*$, because it is independent of $\alpha$.
Thus, for partially scaled inputs $[\alpha x_{1:I}; x_{(I+1):(J)}]^\top$, we can view $g$ as another neural network $g^\ast$ with the same parameters and architecture as $g$, except in the first layer, where $g^*$ only takes $x_{1:I}$ as input and has first layer weights $W_1^{(I)}$ and bias $b_*$. According to Theorem~\ref{thm:hein}, the network $g^*(\alpha x_{1:I})$ produces overconfidence since its inputs are fully scaled. Thus, the same must hold for $g([\alpha x_{1:I}; x_{(I+1):(J)}]^\top)$ and Corollary~\ref{cor:overconfidence} follows.
\end{proof}

\subsubsection{Overconfidence with deterministic meta-learning}
Theorem~\ref{thm:hein} and Corollary~\ref{cor:overconfidence} carry over to the deterministic meta-learner given a context set $\D^{(c)}$. The context encoder $f_\vartheta^{(c)}$ adapts the last-layer weight matrix of the target encoder $f_\psi^{(t)}$, so that the linear regions of the target encoder depend on  $\D^{(c)}$. On a linear region  $Q_l(\D^{(c)})$ and at an arbitrary covariate $x_t$, the target encoder can be expressed as
\begin{align*}
    f_\psi^{(t)}\big|_{Q_l(\D^{(c)})} (x_t) = V_l({\D^{(c)}})x_t + a_l(\D^{(c)}),
\end{align*}
where the coefficients of the affine representation are fixed for a given context set. The context set merely changes the location of the linear regions $\{Q_l\}_{l=1}^R$ and the value of $\alpha$ at which the asymptotic regime is attained. Relating this to our setting, partitioning $x_t = [ \bar{x}_t^\top, t]^\top$ and scaling the time $\alpha \cdot t$ leads to overconfidence through the same reasoning as previously discussed. More generally, this result also extends to CNPs, a state-of-the-art meta-learning architecture, which can be interpreted as a hypernetwork in which the bias of the first layer of the target encoder is predicted by the context encoder.

\subsection{Simulation study}\label{subsec:sim}
We now verify our theoretical results in a simulation study using synthetic MRI volumes and disease trajectories. Because MRI volumes with cerebral atrophy are difficult to generate, we instead consider a neurodegenerative disease in which progression is driven by lesions, such as in multiple sclerosis. Lesions represent damaged areas of brain tissue and are clearly visible as bright regions in T1-weighted MRI volumes. 

\paragraph*{Data generation}
We use a continuous Gaussian process (GP) to model a latent disease state $z_{i,{t_{i,j}}}$ for an individual $i$ at a specific time point $t_{i,j}  \in \Reals$ and latent lesion size $l_{i,t_{i,j}} \in \Reals_0^{+}$. The continuous latent state $z_{i, t_{i,j}}$ is discretized to a discrete disease score $y_{i,{t_{i,j}}} \in \{1, \ldots, K\}$. We use $K = 10$ classes here, as opposed to the $K = 3$ in the real-world application, to demonstrate that, with sufficient training data, the model can accurately capture disease progression across a wide range of severity levels. 
The lesion size evolves as a random walk with a positive drift $d$ over time, as well as flare-ups and remissions
\begin{align*}
    l_{i,t_{i,j}} = l_{i,t_{i,j - 1}} + d + F_{i, t_{i,j}} - R_{i, t_{i,j}} + \varepsilon_{i, t_{i,j}}, 
\end{align*}
where $\varepsilon_{i, t_{i,j}} \sim \mathcal{N}(0, 0.01)$ is i.i.d.~observational noise, the flare-up term $F_{i,t_{i,j}}\sim\mathcal{U}[0,2.5]$ with probability $0.05$ and $0$ otherwise, and the remission term
$R_{i,t_{i,j}}\sim\mathcal{U}[0,1.5]$ with probability $0.02$ and $0$ otherwise. To ensure $l_{i, t_{i,j}} \Reals_0^+$, negative values are set to zero. Example trajectories can be found in Supplementary Material D. 
For each lesion size, we generate a synthetic MRI image by inserting a bright lesion of area $l_{i,t_{i,j}}$ into a pre-generated background image. 
Figure~\ref{fig:gp_with_mri} depicts a sample of the GP over lesion size and time (left) and examples of synthetic MRI volume for several lesion sizes (right).
\begin{figure}[H]
    \centering
    \begin{minipage}[t]{1\linewidth}  
        \centering
        \includegraphics[width=.7\linewidth]{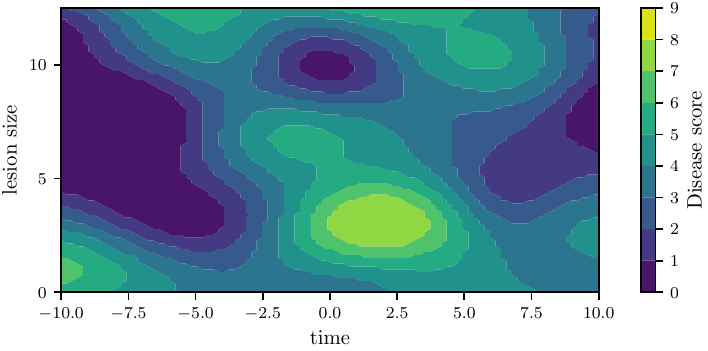}%
        \hspace{0.5cm} 
        \raisebox{.73cm}{\includegraphics[width=.15\linewidth]{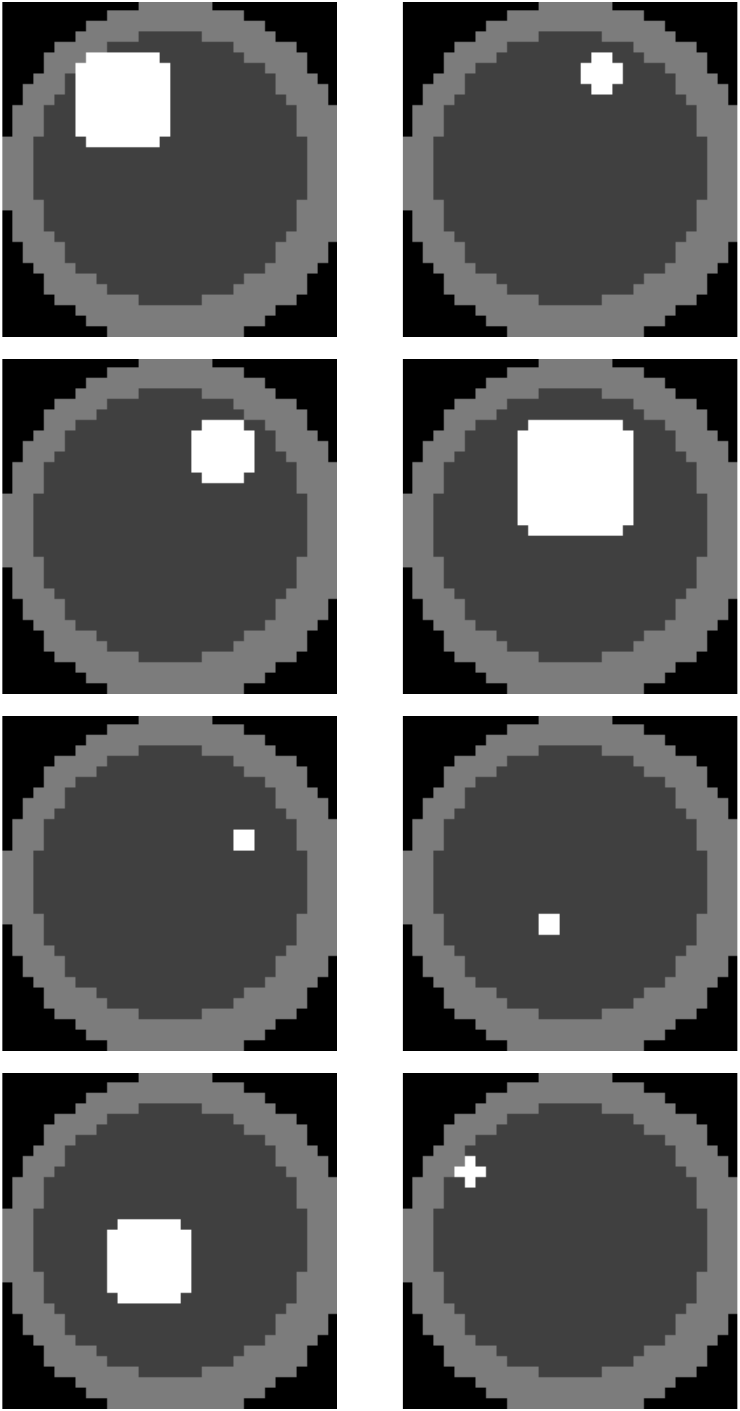}}%
    \end{minipage}
    \caption{Discretized sample from a Gaussian process used to model discrete disease score over time and lesion size and resulting synthetic MRI images (right).}
    \label{fig:gp_with_mri}
\end{figure}
The GP kernel is a product of two squared exponential kernels, corresponding to lesion size and time, respectively. A non-separable kernel could also be used to introduce a positive correlation between lesion size and time. However, our setting accommodates a wider range of progression patterns, in which lesions can grow, stabilize, or vanish over time. The disease trajectory of an individual is obtained by sampling a function from the GP and discretizing the latent disease states.

In line with the real-world setting from the ADNI data, we split the time into an in-sample period (April 2019--June 2020) and an out-of-sample period (January 2015--April 2019 and May 2020--December 2024). 

\paragraph*{Training, prediction and evaluation}
We train the deterministic and Bayesian meta-learners to predict the disease score from the synthetic MRI volumes and time points within the in-sample period. During both training and prediction, we condition on a random number of context points drawn uniformly from the interval $[10,20]$ within the in-sample period. The predictive performance is evaluated both on in-sample and out-of-sample data. Additional information on data generation, architecture, episodic sampling scheme, and training details can be found in Supplementary Material D.

We evaluate both models in terms of accuracy, which is defined as the proportion of correctly classified observations. 
In addition, we evaluate the quality of the predictive probabilities via the average negative log-likelihood (NLL), Brier score, and expected calibration error \citep[ECE;][]{NaeCooHau2015}. For a test set with $N_{\text{val}}$ individuals, the average negative log-likelihood is $\text{NLL} =  -\frac{1}{N_{\text{val}}} \sum_{i=1}^{N_{\text{val}}} \frac{1}{T_i + H_i} \sum_{j=1}^{T_i + H_i} \sum_{k=1}^{K} y_{i,t_{i,j}}^{(k)} \log \hat{p}_{i,t_{i,j}}^{(k)}$,
where $y_{i,{t_{i,j}}}^{(k)} = 1$ if the observation belongs to class $k$ and zero otherwise, and $\hat{p}_{i,t_{i,j}}^{(k)} \in \{\hat{p}_{i,t_{i,j}}^{\text{Det}}(k), \hat{p}_{i,t_{i,j}}^{\text{MC}}(k)\}$ denotes the predicted  probability of class $k$ at observation $(i,j)$.

The Brier score measures the accuracy of the predictive uncertainty by the mean squared difference between the predicted probability and the actual outcome. The Brier score is defined as $\text{Brier} = \frac{1}{N_{\text{val}}} \sum_{i=1}^{N_{\text{val}}}\frac{1}{T_i+H_i}
\sum_{j=1}^{T_i+H_i} \sum_{k=1}^K \left(y_{i,{t_{i,j}}}^{(k)}  - \hat p_{i,{t_{i,j}}}^{(k)}\right)^2$.
The worst value of the Brier score is two, and a perfectly accurate probabilistic prediction produces a Brier score of zero.

Calibration is assessed by the $\text{ECE} = \sum_{m=1}^{M} \frac{|B_m|}{\sum_{i=1}^{N_{\text{val}}} (T_i+H_i)} \left| {\text{acc}(B_m)} - {\text{conf}(B_m)}\right|$,
where $\text{acc}(B_m) = \frac{1}{|B_m|}\sum_{(i,j)\in B_m}\mathds{1}\{\hat{y}_{i,t_{i,j}} = y_{i,t_{i,j}}\}$ denotes the average accuracy and  $\text{conf}(B_m) = \frac{1}{|B_m|}\sum_{(i,j)\in B_m}\max_k \hat{p}_{i,t_{i,j}}^{(k)}$ the average confidence for observations in bin $B_m$ with $m = 1, \ldots, M$.
We set $M = 15$, which allows for sensitivity to miscalibration throughout the entire confidence spectrum. The ECE measures the average deviation between the model's confidence in its prediction and the empirical accuracy. A perfectly calibrated model would produce an ECE of zero, meaning that for the observations within each bin $B_m \in [\alpha_m^{\text{low}},\alpha_m^{\text{high}}]$ the empirical accuracy of the predictions equals the average predicted confidence. Vice versa, the maximum ECE is one.

Table~\ref{tab:synth_benchmark} reports the predictive performance on disease trajectories of $1{,}000$ individuals averaged over 10 independent runs each using a different random seed for initialization. 
\begin{table}[ht!]
\centering
  \caption{Predictive performance on in-sample and out-of-sample data. All values are averages and standard deviations over 10 trials. Bold indicates the best values within a data regime, accounting for one standard deviation interval around the mean.}\label{tab:synth_benchmark}
\begin{tabular}{llcccc}
    \toprule
& \textbf{Model} & \textbf{Acc.} (\%, $\uparrow$) &\textbf{NLL} ($\downarrow$) & \textbf{Brier}($\downarrow$) & \textbf{ECE} ($\downarrow$) \\
    \midrule
\multirow{2}{*}{\textbf{In-sample}} & Meta-learner & 78.08 {\scriptsize$\pm$ 0.31} & 0.61 {\scriptsize$\pm$ 0.01} & 0.32 {\scriptsize$\pm$ 0.00} & 0.03 {\scriptsize$\pm$ 0.00} \\
 & Meta-learner + LLI & 78.07 {\scriptsize$\pm$ 0.31} & 0.61 {\scriptsize$\pm$ 0.01} & 0.32 {\scriptsize$\pm$ 0.00} & 0.04 {\scriptsize$\pm$ 0.00} \\
    \midrule
\multirow{2}{*}{\textbf{Out-of-sample}} & Meta-learner & 13.96 {\scriptsize$\pm$ 0.87} & 8.61 {\scriptsize$\pm$ 0.42} & 1.44 {\scriptsize$\pm$ 0.03} & 0.64 {\scriptsize$\pm$ 0.02} \\
 & Meta-learner + LLI & 13.86 {\scriptsize$\pm$ 0.86} & \textbf{5.45} {\scriptsize$\pm$ 0.56} & \textbf{1.15} {\scriptsize$\pm$ 0.03} & \textbf{0.39} {\scriptsize$\pm$ 0.02} \\
    \bottomrule
\end{tabular}
\end{table}
Both learners have almost identical performance close to the context points (in-sample). Far from the context points (out-of-sample), the accuracy strongly drops by about $65$ percentage points. Consistent with our theoretical results, the deterministic meta-learner is severely miscalibrated with an ECE of $0.64$, while the Bayesian meta-learner improves calibration substantially with an ECE of $0.39$. 

The average NLL and Brier score of the deterministic meta-learner also deteriorate far from the context points, but to a lesser extent for the Bayesian meta-learner. Figure~\ref{fig:synth_example} illustrates the underlying reason.
\begin{figure}[H]
    \centering
    \includegraphics[width=1\linewidth]{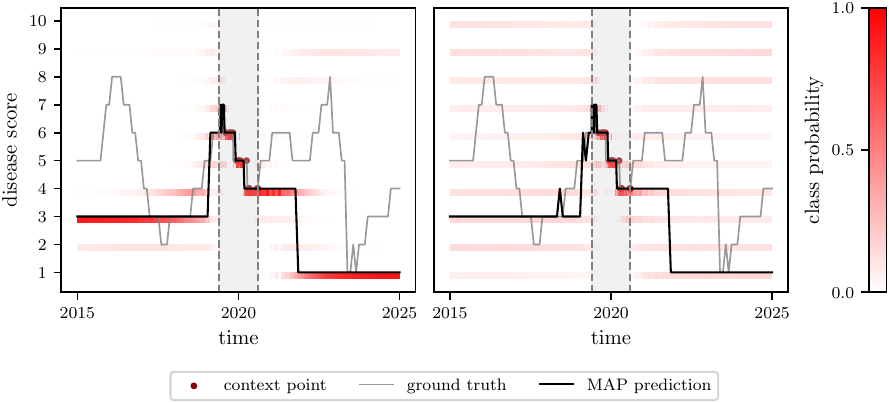}
    \caption{MAP prediction (\pred) and predictive uncertainty (red) of a deterministic meta-learner (left) and Bayesian meta-learner (right) for an observed disease trajectory (\groundtruth) close to the context points (\redcircle; gray shaded area) and extrapolation regime.  The deterministic meta-learner is overconfident (\score) far from the context points, while our Bayesian meta-learner mitigates overconfidence by assigning probability mass more uniformly over disease scores (\scorelight).}
    \label{fig:synth_example}
\end{figure}
The deterministic meta-learner (left) concentrates nearly all mass on a single class far away from the context points. The Bayesian meta-learner (right), on the other hand, distributes probability mass more evenly, so that the observed trajectory is covered with non-zero probability.

\section{Application to disease trajectory modeling}\label{sec:application}
We now evaluate the predictive accuracy and calibration of our proposed Bayesian meta-learner on real-world trajectories from the ADNI database. Following the simulation setup of Section \ref{subsec:sim}, we train the models on a restricted time window and then assess their ability to predict disease scores at future time points using the respective MRI volumes. The training, validation, and test splits are described in Table~\ref{tab:train_test_split}.

\subsection{Benchmarks, preprocessing and evaluation metrics}
We compare the following models:
\begin{itemize}
        \item \textbf{Naive learner}: Classifier that always predicts the majority class (NC) irrespective of the covariates.
        \item \textbf{Single-task learner}: Feedforward neural network that predicts the disease score $y_{i, t_{i,j}}^{(t)}$ from $x_{i,t_{i,j}}^{(t)}$, without using the individual's observed disease trajectory.
        \item \textbf{Single-task learner + LLI}: The previous single-task learner with last-layer inference using a Laplace approximation based on the entire context data $\D^{(c)}$.
        \item \textbf{Meta-learner}: Deterministic meta-learner with the architecture from Section~\ref{sec:architecture} with inference from Section~\ref{sec:det_pred}.
        \item \textbf{Meta-learner + LLI}: Bayesian meta-learner with the architecture from Section~\ref{sec:architecture} and last-layer inference from Section~\ref{sec:bayes_pred}.
    \end{itemize}
All models except the naive learner are feedforward neural networks with ReLU activations in all hidden layers and implemented in PyTorch \citep{PasGroMas2019}.

\paragraph*{Preprocessing and training details}
All MRI volumes are skull-stripped, intensity-normalized, and projected to a standard coordinate space using the \emph{neural-preprocessing} package \citep{HeWanSab2023}. The feature embeddings $\bar{x}_{i, t_{i,j}}$ are extracted with SAM-Med3D \citep{WanGuoYe2025}, a general-purpose foundation model for 3D medical imaging data. To prevent data leakage across training, validation, and test sets, we exclude recent 3D foundation models that have been trained directly on ADNI \citep{TakGarCha2024, KacSzeNic2025, BarBruDuf2024}.
We convert the SAM-Med3D embeddings with dimension $384 \times 8 \times 8 \times 8$ to a $64$-dimensional vector using a convolutional neural network. This network is trained jointly with the single-task learner, and, once trained, it is used by all other models. Further training details are provided in Supplementary Material D.

\paragraph*{Evaluation} To evaluate the accuracy of point predictions, we compute the macro-F1 score, denoted by mF1. The macro-F1 score is an equally weighted average of the class-specific F1 scores denoted by $\text{F1}_k$ for $k=1,\ldots,K$
\begin{align*}
   \text{mF1} = \frac{1}{K}\sum_{k=1}^K \text{F1}_k,\quad\mbox{ where }\quad \text{F1}_k = \frac{2 \cdot \text{precision}_k \cdot \text{recall}_k}{\text{precision}_k +\text{recall}_k}, 
\end{align*}
where $\text{precision}_k$ measures the proportion of predicted class $k$ examples that actually belong to class $k$, and $\text{recall}_k$ measures the proportion of the true class $k$ examples that are correctly predicted. It ranges from 0 to 1, where 0 indicates the worst possible performance, and 1 corresponds to perfect prediction. The macro F1 score is robust to class imbalance, which is particularly important for the ADNI dataset, where individuals with AD are underrepresented despite being the primary group of interest. 

To evaluate predictive uncertainty, we employ the average NLL, the Brier score, and the ECE as described in Section~\ref{subsec:sim}. Note that the naive learner does not produce a predictive distribution and will therefore only be included in the comparison of the macro F1 score. 
Throughout, we only evaluate on observations that have at least one previous observation in the in-sample range (ADNI1; 2006--2008). For the meta-learners, we condition on all prior observations of the same individual within this range. 

\subsection{Results}

\paragraph*{In-sample}
Table~\ref{tab:adni_id_benchmark} reports the performance for test observations within the training window (ADNI1; 2006--2008).
\begin{table}[ht!]
  \centering
  \caption{Predictive performance for the in-sample test set (ADNI1; 2006--2008)}\label{tab:adni_id_benchmark}
  \begin{tabular}{lccccccc}
    \toprule
    \textbf{Model} & \textbf{mF1} (\%, $\uparrow)$ & \textbf{NLL} $(\downarrow)$ & \textbf{Brier} $(\downarrow)$ & \textbf{ECE} $(\downarrow)$ \\
    \midrule
    Naive learner &    19.60 & - & - & - \\
    Single-task learner &  48.58 & 1.07 &  0.62  &  \textbf{0.10}\\
    Single-task learner + LLI &  48.58 & \textbf{0.88} &  0.62  &  \textbf{0.10}\\
    Meta-learner (ours) &  \textbf{50.82}  & 0.91 & \textbf{0.53} &  \textbf{0.10} \\
    Meta-learner + LLI (ours)   &  \textbf{50.82}  & 0.91 & \textbf{0.53} &  0.11 \\
    \bottomrule
  \end{tabular}
\end{table}
Performance is generally similar across all models and metrics, indicating that all learners can capture the disease trajectories within the training window well. Only the naive learner, which always predicts the majority class, performs poorly. Since the temporal range of the in-sample data is short and the disease scores do not vary much within this timeframe, including the historical disease trajectory with the meta-learner yields only a modest improvement in the F1- and Brier scores.

\paragraph*{Out-of-sample}
Table~\ref{tab:adni34_benchmark} reports the predictive performance for the years 2016--2022 (ADNI3) and 2023--2025 (ADNI4). The naive learner is omitted, as it is uninformative for out-of-distribution prediction and can occasionally appear competitive due to class imbalance. As expected, predictive performance worsens for all models when predicting beyond the training window, but the meta-learners perform substantially better than the single-task learners. Consistent with our theoretical results, the single-task learner is severely miscalibrated with ECEs of $0.52$ (ADNI3) and $0.55$ (ADNI4). Applying last-layer inference improves calibration by reducing the ECE to $0.38$ and $0.34$, respectively. The deterministic meta-learner is better calibrated than the single-task models with ECEs of $0.22$ (ADNI3) and $0.28$ (ADNI4). Overall, our
Bayesian meta-learner achieves the best calibration, reaching ECE values of $0.18$ (ADNI3, ADNI4), which are only about $0.08$ above those of the in-sample test set. 

\begin{table}[ht!]
  \centering
  \caption{Predictive performance for the out-of-sample test sets (ADNI3 and ADNI4).}\label{tab:adni34_benchmark}
  \begin{tabular}{lcccc}
    \toprule
     & \multicolumn{4}{c}{\textbf{ADNI3 (2016-2022)}}  \\
     \midrule
    \textbf{Model} & \textbf{mF1} (\%, $\uparrow)$ & \textbf{NLL} $(\downarrow)$ & \textbf{Brier} $(\downarrow)$ & \textbf{ECE} $(\downarrow)$ \\
    \midrule
    Single-task learner & 6.03  & 1.51  & 1.00  & 0.52  \\
    Single-task learner + LLI   & 5.89 & 1.18 & 0.78  & 0.38  \\
    Meta-learner (ours) & \textbf{27.10}  & \textbf{0.74} & \textbf{0.44} & 0.22  \\
    Meta-learner + LLI (ours)  & \textbf{27.10} & \textbf{0.74}  & \textbf{0.44} & \textbf{0.18}  \\
     \\
    & \multicolumn{4}{c}{\textbf{ADNI4 (2022-2025)}}  \\
    \midrule
     & \textbf{mF1} (\%, $\uparrow)$ & \textbf{NLL} $(\downarrow)$ & \textbf{Brier} $(\downarrow)$ & \textbf{ECE} $(\downarrow)$ \\
    \midrule
    Single-task learner &  5.56 & 1.67  & 1.16  & 0.55 \\
    Single-task learner + LLI   &  6.06 & 1.12 & 0.77 & 0.34 \\
    Meta-learner (ours)  & \textbf{25.93} & \textbf{0.79} & \textbf{0.46} & 0.28\\
    Meta-learner + LLI (ours)   & \textbf{25.93} & \textbf{0.79}  & 0.47 & \textbf{0.18} \\
    \bottomrule
  \end{tabular}
\end{table}

Overall, the meta-learners achieve better prediction accuracy and calibration by accounting for the historical disease trajectory. 
In particular, our proposed Bayesian meta-learner with last-layer inference yields an additional improvement in calibration, producing more reliable uncertainty estimates even when predicting far into the future. 

\section{Discussion}\label{sec:discussion}
We have proposed a novel Bayesian meta-learner combined with last-layer inference to model the progression of AD based on irregularly sampled, variable-length MRI time series and disease severity scores.
In an extensive benchmark study on simulated and real-world data, the Bayesian meta-learner achieves improved generalization and uncertainty quantification compared to its deterministic counterpart and single-task models, in particular when predicting into the future and outside the range of observed data.

Regarding Bayesian inference, future research could investigate performing last-layer inference with count models (e.g., Poisson or negative binomial regression), as they account for the natural ordering of discrete disease scores. The Poisson regression model can be implemented in the last layer by predicting the rate parameter of the Poisson distribution as the final output using an exponential function as the activation function. Similarly, for the negative binomial regression, the mean and dispersion parameters can be predicted with exponential activations. However, this can lead to unstable training, because both the activation function itself and its derivative are exponential, which may cause exploding gradients (i.e., a large increase in the norm of the gradient). Potential solutions include gradient clipping or replacing the exponential with a softplus activation \citep{DugBenBel2000}, though the latter breaks the exact correspondence with the canonical formulation of the original regression models and complicates inference.
An alternative is to model the disease scores in the Bayesian meta-learner using a Dirichlet prior, which results in a categorical distribution for the disease scores and admits closed-form training and inference \citep{MalGal2018, HarWilSno2023}.

Throughout the paper, we have focused on the case in which the MRI volumes for the target set are available. An important extension for practical use is to predict the entire future disease trajectory from the historical disease trajectory, without requiring access to future MRI volumes. This can be achieved by generating MRI volumes, for example, with diffusion models for MRI \citep{KhaMulTay2023}.

Finally, further modifications for the architecture remain to be explored. For instance, replacing the mean aggregator with attention-based aggregation \citep{KimMniSch2019} could allow the model to better capture dependencies between context points when learning the last-layer weight matrix. In addition, more extensive adaptation in earlier layers may be beneficial in settings with longer or more complex disease progression. Our results, however, together with those from similar meta-learning architectures such as CNAPs and its Bayesian counterpart VERSA (\citealp{GorBronBau2019}), suggest that adapting only the last-layer weights provides sufficient flexibility for conditioning on the historical disease trajectory.

\vspace{10pt}

\textbf{Acknowledgments:} The authors acknowledge funding by the Deutsche Forschungsgemeinschaft (DFG, German Research Foundation) through the grant KL3037/8-1 (KI-FOR 5363 DeSBi) and the Emmy Noether grant KL 3037/1-1. The authors thank Giovanni Roman\`{o} and Dongjae Son for proofreading and helpful comments on an earlier version of this manuscript.

Data collection and sharing for the Alzheimer's Disease Neuroimaging Initiative (ADNI) is funded by the National Institute on Aging (National Institutes of Health Grant U19AG024904). The grantee organization is the Northern California Institute for Research and Education. In the past, ADNI has also received funding from the National Institute of Biomedical Imaging and Bioengineering, the Canadian Institutes of Health Research, and private sector contributions through the Foundation for the National Institutes of Health (FNIH) including generous contributions from the following: AbbVie, Alzheimer’s Association; Alzheimer’s Drug Discovery Foundation; Araclon Biotech; BioClinica, Inc.; Biogen; BristolMyers Squibb Company; CereSpir, Inc.; Cogstate; Eisai Inc.; Elan Pharmaceuticals, Inc.; Eli Lilly and Company; EuroImmun; F. Hoffmann-La Roche Ltd and its affiliated company Genentech, Inc.; Fujirebio; GE Healthcare; IXICO Ltd.; Janssen Alzheimer Immunotherapy Research \& Development, LLC.; Johnson \& Johnson Pharmaceutical Research \& Development LLC.; Lumosity; Lundbeck; Merck \& Co., Inc.; Meso Scale Diagnostics, LLC.; NeuroRx Research; Neurotrack Technologies; Novartis Pharmaceuticals Corporation; Pfizer Inc.; Piramal Imaging; Servier; Takeda Pharmaceutical Company; and Transition Therapeutics.

\bibliography{references}

\end{document}

%% file: definitions.tex
\DeclareMathOperator*{\argmin}{arg\,min}
\mathchardef\perp="323F

\newcommand{\numclasses}{K} 
\newcommand{\dimlastlayerweights}{F} 
\newcommand{\dimimg}{D}  
\newcommand{\MN}{\mathcal{MN}} 

\newcommand{\D}{\mathcal{D}}

\newcommand{\argmax}{\operatornamewithlimits{arg \, max}}

\newcommand{\Reals}{\mathbb{R}}

\newcommand{\groundtruth}{{\color{gray!80}\rule[0.5ex]{0.3cm}{1pt}}}
\newcommand{\pred}{\rule[0.5ex]{0.3cm}{1.5pt}}
\newcommand{\score}{{\color{red}\rule[-0ex]{0.5cm}{6pt}}}
\newcommand{\scorelight}{{\color{red!15}\rule[-0ex]{0.5cm}{6pt}}}
\definecolor{darkred}{RGB}{139,0,0}
\newcommand{\redcircle}{{\color{darkred}\raisebox{0.ex}{\large$\bullet$}}}